\documentclass{article}

\usepackage{PRIMEarxiv}

\usepackage[utf8]{inputenc} 
\usepackage[T1]{fontenc}    
\usepackage{hyperref}       
\usepackage{url}            
\usepackage{booktabs}       
\usepackage{amsfonts}       
\usepackage{nicefrac}       
\usepackage{microtype}      
\usepackage{lipsum}
\usepackage{fancyhdr}       
\usepackage{graphicx}       
\graphicspath{{media/}}     

\usepackage{amsmath}
\usepackage{amsthm}  
\usepackage{color}
\usepackage{comment}
\usepackage{xspace}
\usepackage{pifont}  
\usepackage{thmtools}  

\newtheorem{thm}{Theorem}

\newtheorem{lem}{Lemma}

\usepackage{stmaryrd}       
\usepackage{xspace}
\makeatletter
\DeclareRobustCommand\onedot{\futurelet\@let@token\@onedot}
\def\@onedot{\ifx\@let@token.\else.\null\fi\xspace}

\def\eg{{e.g}\onedot} 
\def\ie{{i.e}\onedot}

\makeatother

\newcommand{\method}{AOL\xspace}
\newcommand{\aolconv}{\method Conv}
\newcommand{\aolfc}{\method FC}
\newcommand{\mm}{MaxMin}

\newcommand{\cra}{certified robust accuracy\xspace}
\newcommand{\Cra}{Certified robust accuracy\xspace}

\newcommand{\supp}{appendix\xspace}

\newcommand{\myparagraph}[1]{\medskip\noindent\textit{#1}}

\newcommand{\tableheader}{Certified Robust Accuracy}
\newcommand{\fulltableheader}{ %
Method      & Standard  & \multicolumn{4}{c}{\tableheader} \\ %
            & Accuracy  & $\epsilon=\frac{36}{255}$ & $\epsilon=\frac{72}{255}$ & $\epsilon=\frac{108}{255}$ & $\epsilon=1$ \\%
}
\newcommand{\tabledesc}{ %
We report the standard accuracy on the test set as well as the \cra under input perturbations %
up to size $\epsilon$ for different values of $\epsilon$. %
}

\newcommand{\atimes}{}  

\newcommand{\spec}[1]{\|#1\|_{\text{spec}}}

\title{Almost-Orthogonal Layers for Efficient General-Purpose Lipschitz Networks
}

\author{
  Bernd Prach, Christoph H. Lampert \\
  Institute of Science and Technology Austria (ISTA) \\
  \texttt{\{bprach,chl\}@ist.ac.at} \\
}

\begin{document}
\maketitle

\begin{abstract}
It is a highly desirable property for deep networks to be robust against small input changes. 
One popular way to achieve this property is by designing networks with a small Lipschitz constant. 
In this work, we propose a new technique for constructing such \emph{Lipschitz networks} that has a number of desirable properties: it can be applied to any linear network layer (fully-connected or convolutional), it provides formal guarantees on the Lipschitz constant, it is easy to implement and efficient to run, and it can be combined with any training objective and optimization method. 
In fact, our technique is the first one in the literature that achieves all of these properties simultaneously.

Our main contribution is a rescaling-based weight matrix parametrization that guarantees each network layer to have a Lipschitz constant of at most $1$ and results in the learned weight matrices to be close to orthogonal.
Hence we call such layers \emph{almost-orthogonal Lipschitz (\method)}.

Experiments and ablation studies in the context of 
image classification with \cra confirm that 
\method layers achieve results that are on par with most existing methods. 
Yet, they are simpler to implement 
and more broadly applicable, because they do not 
require computationally expensive matrix 
orthogonalization or inversion steps as 
part of the network architecture.

We provide code at 
\url{https://github.com/berndprach/AOL}.
%
%
\end{abstract}


\section{Introduction}
Deep networks are often the undisputed state of the art when it comes to solving computer vision tasks with high accuracy.  
%
However, the resulting systems tend to be not very \emph{robust}, \eg, against small changes in the input data.
This makes them untrustworthy for safety-critical high-stakes tasks, such as autonomous driving or medical diagnosis.

A typical example of this phenomenon are \emph{adversarial examples}~\cite{Szegedy_2014_ICLR}: 
imperceptibly small changes to an image can drastically change the outputs of a deep learning 
classifier when chosen in an adversarial way. 
Since their discovery, numerous methods were developed to make networks more 
robust against adversarial examples. However, in response a comparable number of new attack forms were found, leading to an ongoing cat-and-mouse game. For surveys on the state of research, see, \eg, \cite{chakraborty2018adversarial,serban2020adversarial,xu2020adversarial}.

A more principled alternative is to create deep networks that are 
robust by design, for example, by restricting the class of functions
they can represent. 
Specifically, if one can ensure that a network has a small 
\emph{Lipschitz constant}, then one knows that small changes 
to the input data will not result in large changes to the output, 
even if the changes are chosen adversarially. 

A number of methods for designing such \emph{Lipschitz networks} have been proposed in the literature, which we discuss in  Section~\ref{sec:relatedwork}. 
However, all of them have individual limitations.
In this work, we introduce the \method 
(for \emph{almost-orthogonal Lipschitz}) method.
It is the first method for constructing 
Lipschitz networks that simultaneously meets 
all of the following desirable criteria:

\medskip\noindent\textbf{Generality.} 
    \method is applicable to a wide range of network architectures, in particular 
    most kinds of fully-connected and convolutional layers. 
    In contrast, many recent methods work only 
    for a restricted set of layer types, such 
    as only fully-connected layers or only 
    convolutional layers with non-overlapping receptive fields.

\medskip\noindent\textbf{Formal guarantees.} \method provably guarantees a Lipschitz constant $1$. This is in contrast to methods that only encourage small Lipschitz constants, \eg, by regularization.

\medskip\noindent\textbf{Efficiency.} \method causes only a small computational overhead at training time and none at all at prediction time. This is in contrast to methods that embed expensive iterative operations such as matrix orthogonalization or inversion steps into the network layers.

\medskip\noindent\textbf{Modularity.} \method can be treated as a black-box module and combined with arbitrary training objective functions and optimizers.
This is in contrast to methods that achieve the Lipschitz property only when combined with, \eg, specific loss-rescaling or projection steps 
during training.

\medskip\method's name stems from the fact that the weight matrices it learns are approximately orthogonal. 
In contrast to prior work, this property is not 
enforced explicitly, which would incur a 
computational cost. Instead, almost-orthogonal 
weight matrices emerge organically during network training. 
The reason is that \method's rescaling step relies 
on an upper bound to the Lipschitz constant that 
is tight for parameter matrices with orthogonal 
columns. 
During training, matrices without that property 
are put at the disadvantage of resulting in outputs of smaller dynamic range. As a consequence, orthogonal matrices are able to achieve smaller values of the loss and are therefore preferred by the optimizer.

\section{Notation and Background}\label{sec:notation}
A function $f:\mathbb{R}^n\to\mathbb{R}^{m}$ is 
called \emph{$L$-Lipschitz continuous} with respect to 
norms $\|.\|_{\mathbb{R}^n}$ and  $\|.\|_{\mathbb{R}^m}$, if it fulfills
\begin{align}
    \|f(x)-f(y)\|_{\mathbb{R}^m} \leq L\|x-y\|_{\mathbb{R}^n},
\end{align}
for all $x$ and $y$,
where $L$ is called the \emph{Lipschitz constant}. 
In this work we only consider Lipschitz-continuity with respect to the Euclidean norm, $\|.\|_2$, and mainly for $L=1$. For conciseness of notation, we refer to such 1-Lipschitz continuous functions simply as \emph{Lipschitz functions}.

For any linear (actually affine) function $f$, the Lipschitz property can be verified by checking if the function's Jacobian matrix, $J_f$, has \emph{spectral norm} $\|J_f\|_{\text{spec}}$ less or equal to $1$, where 
\begin{align}
    [J_f]_{ij}=\frac{\partial f_i}{\partial x_j}
    \qquad\text{and}\qquad
    \|M\|_{\text{spec}} = \max_{\|v\|_2=1} \|Mv\|_2.
\end{align}
The spectral norm of a matrix $M$ can in fact be computed numerically as it is 
identical to the largest singular value of the matrix. 
This, however, typically requires iterative algorithms that are computationally expensive in high-dimensional settings.
An exception is if $M$ is an orthogonal matrix, \ie $M^\top M=I$, for $I$ the identity matrix. 
In that case we know that all its singular values are $1$ and $\|Mv\|_2=\|v\|_2$ for all $v$, so the corresponding linear transformation is Lipschitz. 

Throughout this work we consider a deep neural network as a concatenation of linear layers (fully-connected or convolutional) alternating with non-linear activation functions. 
We then study the problem how to ensure that the resulting network function is Lipschitz. 

It is known that computing the exact Lipschitz constant of a neural network is an NP-hard problem~\cite{virmaux2018lipschitz}.
However, upper bounds can be computed more efficiently, 
\eg, by multiplying the individual Lipschitz constants of 
all layers.

\section{Related work}\label{sec:relatedwork}
The first attempts to train deep networks with small 
Lipschitz constant used ad-hoc techniques, such as weight clipping~\cite{arjovsky2017wasserstein} or 
regularizing either the network gradients~\cite{gulrajani2017improved}
or the individual layers' spectral norms~\cite{miyato2018spectral}.
However, these techniques do not formally guarantee 
bounds on the Lipschitz constant of the trained 
network. 
Formal guarantees are provided by constructions 
that ensure that each individual network layer 
is Lipschitz. 
Combined with Lipschitz activation functions, such as
ReLU, MaxMin or $\tanh$, this ensures that the overall 
network function is Lipschitz.

In the following, we discuss a number of prior methods for obtaining Lipschitz networks. 
A structured overview of their properties can be found in Table~\ref{table:methods}.

\newcommand{\Y}{\textcolor{green}{\ding{51}}}
\newcommand{\N}{\textcolor{red}{\ding{55}}}
\newcommand{\M}[1]{$\sim^{#1}$}

\begin{table}[t]
\setlength{\tabcolsep}{4pt}
\begin{center}
\caption{Overview of the properties of different methods for learning Lipschitz networks. 
Columns indicate: \emph{E (efficiency)}: no internal iterative procedure required, scales well with the input size. 
\emph{F (formal guarantees)}: provides a guarantee about the Lipschitz constant of the trained network.
\emph{G (generality)}: can be applied to fully-connected as well as convolutional layers. 
\emph{M (modularity)}: can be used with any training objective and optimization method. 
\\
$\sim$ symbols indicate that a property is partially fulfilled. Superscripts provide further explanations:
$^1$ requires a regularization loss.
$^2$ internal methods would have to be run to convergence.
$^3$ iterative procedure that can be split between training steps.
$^4$ requires matrix orthogonalization.
$^5$ requires inversion of an input-sized matrix.
$^6$ requires circular padding and full-image kernel size. 
$^7$ requires large kernel sizes to ensure orthogonality.
}
\label{table:methods}
\begin{tabular}
{l|p{.05\textwidth}p{.05\textwidth}p{.05\textwidth}p{.05\textwidth}p{.3\textwidth}}
\hline\noalign{\smallskip}
Method  & E & F & G & M & Methodology \\
\noalign{\smallskip}\hline\noalign{\smallskip}
WGAN~\cite{arjovsky2017wasserstein}        & \Y    & \Y  & \Y  &\N   & \mbox{weight clipping}\\
WGAN-GP~\cite{gulrajani2017improved}  & \Y    &\N  &\Y   & \M{1}  & regularization\\
Parseval Networks~\cite{Cisse_2017_ICML}    & \Y    & \N  & \Y    &\M{1}& regularization \\
OCNN \cite{Wang_2020_CVPR}                  & \Y    & \N  & \Y    &\M{1}& regularization \\
SN~\cite{miyato2018spectral}     &\N& \M{2}  & \Y  &\Y& parameter rescaling \\
LCC~\cite{Gouk_2021_ML}              & \N    & \M{2}  & \Y & \N  & \mbox{parameter rescaling} \\
LMT~\cite{Tsuzuku_2018_NIPS}                & \M{3} & \N  & \Y    & \N  & \mbox{loss rescaling} \\
GloRo~\cite{Leino_2021_ICML}                & \M{3} &\M{2}& \Y    & \N  & \mbox{loss rescaling} \\
BCOP~\cite{Li_2019_NIPS}                    & \N$^4$& \Y  & \Y    & \Y  & explicit orthogonalization \\
GroupSort\cite{Anil_2019_ICML}            & \N$^4$&\M{2}& \N    & \Y  & explicit orthogonalization \\
ONI~\cite{Huang_2020_CVPR}                  & \N    &\M{2}& \Y    & \Y  & explicit orthogonalization \\
Cayley Convs~\cite{Trockman_2021_ICLR}      & \M{5} & \Y  & \M{6} & \Y  & explicit orthogonalization \\
SOP~\cite{Singla_2021_ICML}                 & \N    &\M{2}& \M{7} & \Y  & explicit orthogonalization \\
ECO~\cite{Yu_2022_ICLR}                     & \N$^4$& \Y  & \M{6} & \Y  & explicit orthogonalization \\
AOL (proposed)                              & \Y    & \Y  & \Y    & \Y  & parameter rescaling\\
\hline
\end{tabular}
\end{center}
\setlength{\tabcolsep}{1.4pt}
\end{table}


\subsection{Bound-based methods}
The Lipschitz property of a network layer could be achieved trivially: one simply computes the layer's Lipschitz constant, or an upper bound, and divides the layer weights by that value.
Applying such a step after training, however, does not lead to satisfactory results in practice, because the dynamic range of the network outputs is reduced by the product of the scale factors. This can be seen as a reduction of network capacity that prevents the network from fitting the training data well. 
Instead, it makes sense to incorporate the Lipschitz condition already at training time, such that the optimization can attempt to find weight matrices that lead to a network that is not only Lipschitz but also able to fit the data well.

The \emph{Lipschitz Constant Constraint (LCC)} method~\cite{Gouk_2021_ML} identifies all weight matrices with spectral norm above a threshold 
$\lambda$ after each weight update and rescales those matrices to have a spectral norm of exactly $\lambda$.
\textit{Lipschitz Margin Training (LMT)}~\cite{Tsuzuku_2018_NIPS} 
and \textit{Globally-Robust Neural Networks (GloRo)}~\cite{Leino_2021_ICML} 
approximate the overall Lipschitz constant from numeric estimates of the largest singular values of the layers' weight matrices.
They integrate this value as a scale factor into their respective loss functions.

In the context of deep learning, controlling only the Lipschitz constant of each layer separately has some drawbacks.
In particular, the product of the individual Lipschitz 
constants might grossly overestimate the network's 
actual Lipschitz constant. 
The reason is that the Lipschitz constant of a layer is determined by a single vector direction of maximal expansion.
When concatenating multiple layers, their directions of maximal expansion will typically
not be aligned, especially with in between nonlinear activations.
As a consequence, the actual maximal amount of expansion will be smaller than the product of the per-layer maximal expansions.
This causes the variance of the activations 
to shrink during the forward pass through the network, even though in principle a sequence of 1-Lipschitz operations could perfectly preserve it. 
Analogously, the magnitude of the gradient signal  shrinks with each layer during the backwards pass of backpropagation training, which can lead to vanishing gradient problems.

\subsection{Orthogonality-based method}
%
%
A way to address the problems of variance-loss and 
vanishing gradients is to use network layers that 
encode \emph{orthogonal} linear operations.
These are 1-Lipschitz, so the overall network will also have that property. However, they are also \emph{isotropic}, in the sense that they preserve data variance and gradient magnitude in all directions, not just a single one. 

For fully-connected layers, it suffices to ensure that the weight matrices themselves are orthogonal.
The \emph{GroupSort}~\cite{Anil_2019_ICML} architecture
achieves this using classic results from numeric analysis~\cite{Bjorck_1971_SIAM}. The authors parameterize an orthogonal weight matrix as a 
specific matrix power series, which they embed 
in truncated form into the network architecture. 
\textit{Orthogonalization by Newton's Iterations (ONI)}~\cite{Huang_2020_CVPR} parameterizes 
orthogonal weight matrices as $\left(V V^\top \right)^{-1/2} V$ for a general parameter matrix $V$. 
As an approximate representation of the inverse operation the authors embed a number of steps of Newton's method into the network. 
Both methods, GroupSort and ONI, have the shortcoming that their orthogonalization schemes require the application of iterative computation schemes which incur a trade-off between the approximation quality and the computational cost.

For convolutional layers, more involved constructions
are required to ensure that the resulting linear 
transformations are orthogonal. 
In particular, enforcing orthogonal kernel matrices is not sufficient in general to ensure a Lipschitz constant of 1 when the convolutions have overlapping receptive fields.

\textit{Skew Orthogonal Convolutions (SOC)}~\cite{Singla_2021_ICML} parameterize orthogonal matrices as the matrix exponentials of skew-symmetric matrices. 
They embed a truncation of the exponential's power  series into the network and bound the resulting error. 
However, SOC requires a rather large number of iterations to yield good approximation quality, which leads to high computational cost.

\textit{Block Convolutional Orthogonal Parameterization (BCOP)}~\cite{Li_2019_NIPS} relies on a matrix decomposition approach to address the problem of orthogonalizing convolutional layers.
The authors parameterize each convolution kernel of size $k \times k$ by a set of $2k-1$ convolutional matrices of size $1\times 2$ or $2\times 1$. These are combined with 
a final pointwise convolution with orthogonal kernel.
However, BCOP also incurs high computation cost, because each of the smaller transforms requires orthogonalizing a corresponding parameter matrix. 

\emph{Cayley Layers} parameterize orthogonal matrices using the Cayley transform~\cite{Cayley_1846_JRAM}. 
Naively, this requires the inversion of a matrix of 
size quadratic in the input dimensions.
However, in~\cite{Trockman_2021_ICLR} the author demonstrate that in certain situations, namely for full image size convolutions with circular padding, the computations can be performed more efficiently. 

\textit{Explicitly Constructed Orthogonal Convolutions} (ECO)~\cite{Yu_2022_ICLR} rely on a theorem that 
relates the singular values of the Jacobian of a
circular convolution to the singular values of a 
set of much smaller matrices~\cite{Sedghi_2019_ICLR}. 
The authors derive a rather efficient parameterization that, however, 
is restricted to full-size dilated convolutions with non-overlapping receptive fields. 

The main shortcomings of Cayley layers and ECO are their restriction to certain full-size convolutions. Those are incompatible with most well-performing network architectures for high-dimensional data, which use
local kernel convolutions, such as $3\times 3$, and overlapping receptive fields.

\subsection{Relation to \method}
The \method method that we detail in Section~\ref{sec:method} can be seen as a hybrid of bound-based and orthogonality-based approaches. 
It mathematically guarantees the Lipschitz property of each network layer by rescaling the corresponding parameter matrix (column-wise for fully-connected layers, 
channel-wise for convolutions). 
In contrast to other bound-based approaches it does not use a computationally expensive iterative approach to estimate the Lipschitz constant as precisely as possible, but it relies on 
a closed-form upper bound. 
The bound is tight for matrices with orthogonal columns.
During training this has the effect that orthogonal parameter matrices are implicitly preferred by the optimizer, because they allow fitting the data, and therefore minimizing the loss, the best. 
Consequently, \method benefits from the advantages of 
orthogonality-based approaches, such as preserving the 
variance of the activations and the gradient magnitude, 
without the other 
methods' shortcomings of requiring difficult parameterizations and being restricted to specific layer types. 

\section{Almost-Orthogonal Lipschitz (AOL) Layers}~\label{sec:method}
In this section, we introduce our main contribution, almost-orthogonal Lipschitz (AOL) layers, which combine the advantages of rescaling and orthogonalization approaches. 
Specifically, we introduce a weight-dependent rescaling technique for the weights of a linear neural network layer that guarantees the layer to be 1-Lipschitz. 
It can be easily computed in closed form and is applicable to fully-connected as well as convolutional layers.

The main ingredient is the following theorem, which 
provides an elementary formula for controlling the 
spectral norm of a matrix by rescaling its 
columns.

\begin{thm} \label{thm}
    For any matrix $P \in \mathbb{R}^{n \times m}$, define $D \in \mathbb{R}^{m \times m}$
    as the diagonal matrix with $D_{ii} = \left( \sum_j \left| P^\top P \right|_{ij} \right)^{-1/2}$ if the expression in the brackets is non-zero, or $D_{ii}=0$ otherwise. 
    Then the spectral norm of $PD$ is bounded by $1$.
\end{thm}

\begin{proof}
The upper bound of the spectral norm of $PD$ follows from an elementary computation. By definition of the spectral norm, we have
\begin{align}
    \spec{PD}^2
    &= \max_{\|\vec{v}\|_2 =1} \| PD \vec{v} \|_2^2
    = \max_{\|\vec{v}\|_2 =1} \vec{v}^\top\!D^\top\! P^\top\!P D \vec{v}.
    \label{eq:spectralnorm}
\end{align}
We observe that for any symmetric matrix $M\in\mathbb{R}^{n\times n}$ and any $w\in\mathbb{R}^n$:
\begin{align}
   \vec{w}^T M \vec{w}
    &\leq \sum_{i,j=1}^n |M_{ij}| |w_i||w_j|
    \leq \sum_{i,j=1}^n \frac{1}{2} |M_{ij}| (w_i^2 + w_j^2) = \sum_{i=1}^n \Big(\sum_{j=1}^n |M_{ij}|\Big) w_i^2 \label{eq:upper_bound}
\end{align}
where the second inequality follows from the general relation $2ab\leq a^2+b^2$. 
Evaluating \eqref{eq:upper_bound} for $M=P^\top\!P$ and $\vec{w}=D\vec{v}$, we obtain for all $\vec{v}$ with $\|\vec{v}\|_2=1$\!\!\!
\begin{align}
 \vec{v}^\top\!D^\top\! P^\top\!P D \,\vec{v}
 &\leq \sum_{i=1}^n \Big(\sum_{j=1}^n |P^\top P|_{ij}\Big)(D_{ii} v_i)^2
 \leq \sum_{i=1}^n v_i^2 = 1
 \label{eq:spectral2}
\end{align}
which proves the bound.
\end{proof}

Note that when $P$ has orthogonal columns of full rank, we have 
that $P^\top P$ is diagonal and $D=(P^\top P)^{-1/2}$, so 
$D^\top P^\top P D=I$, for $I$ the identity matrix. 
Consequently, \eqref{eq:spectral2} 
holds with equality and the bound in Theorem~\ref{thm} is tight.

\medskip
In the rest of this section, we demonstrate how Theorem~\ref{thm} allows us to control 
the Lipschitz constant of any linear layer in a neural network. 

\subsection{Fully-Connected Lipschitz Layers}
We first discuss the case of fully-connected layers.
\begin{lem}[Fully-Connected \method Layers] \label{lem:fullyconnected}
Let $P\in \mathbb{R}^{n\times m}$ be an arbitrary parameter matrix. 
Then, the fully-connected network layer 
\begin{align}
f(x) = W x + b\label{eq:linearlayer}
\end{align}
is guaranteed to be 1-Lipschitz, when $W = PD$ for $D$ defined as in Theorem~\ref{thm}.
\end{lem}
\begin{proof}
The Lemma follows from Theorem~\ref{thm},
because the Lipschitz constant of $f$ is bounded by the spectral norm of its Jacobian matrix, which is simply $W$.
\end{proof}

\paragraph{Discussion.} 
Despite its simplicity, there are a number aspects of Lemma 1 that are worth a closer look.
First, we observe that a layer of the form $f(x)=PDx+b$ can be interpreted in two ways, depending on how we (mentally) put brackets into the linear term.
In the form $f(x)=P(Dx)+b$, we apply an arbitrary weight matrix to a suitably rescaled input vector. In the form $f(x)=(PD)x+b$, we apply a column-rescaling operation to the weight matrix before applying it to the unchanged input. 
The two views highlight different aspects of \method. 
The first view reflects the flexibility and high capacity of learning with an arbitrary parameter matrix, with only an intermediate rescaling operation to prevent the growth of 
the Lipschitz constant.
The second view shows that \method layers can be implemented without any overhead at prediction time, because the rescaling factors can be absorbed in the parameter matrix itself, even preserving potential structural properties such as sparsity patterns.

As a second insight from Lemma~\ref{lem:fullyconnected} we obtain how \method relates to prior methods that rely on orthogonal weight matrices. 
As derived after Theorem~\ref{thm}, if the parameter matrix, $P$, has orthogonal columns of full rank, then $W=PD$ is an orthogonal weight matrix. 
In particular, when $P$ is already an orthonormal matrix,
then $D$ will be the identity matrix,
and $W$ will be equal to $P$.
Therefore, our method can express any linear map 
based on an orthonormal matrix, 
but it can also express other linear maps.
If the columns of $P$ are approximately orthogonal, in the sense that $P^\top\!P$ is approximately diagonal, then the entries of $D$ are dominated by the diagonal entries of the product. The multiplication by $D$ acts mostly as a normalization of the length of the columns of $P$, and the resulting $W$ is an almost-orthogonal matrix.

Finally, observe that Lemma~\ref{lem:fullyconnected} 
does not put any specific numeric or structural constraints on the parameter matrix.
Consequently, there are no restrictions on the optimizer or objective function when training \method-networks. 

\subsection{Convolutional Lipschitz Layers}
An analog of Lemma~\ref{lem:fullyconnected} for 
convolutional layers can, in principle, be obtained by applying the same construction as above: convolutions are 
linear operations, so we could compute their Jacobian 
matrix and determine an appropriate rescaling matrix 
from it.
However, this naive approach would be inefficient, because it would require working with matrices that are of a size quadratic in the number of input dimensions and channels.
Instead, by a more refined analysis, we obtain the following result.

\begin{restatable}[Convolutional \method Layers]{lem}{convlemma}
\label{lem:convolutional}

Let $P\in\mathbb{R}^{k \times k \times  c_\text{I} \times c_\text{O}}$, be a convolution kernel matrix, where $k \times k$ is the \emph{kernel size} and $c_\text{I}$ and $c_\text{O}$ are the number of input and output channels, respectively. %
Then, the convolutional layer 
\begin{align}
f(x) = P * R(x) + b
\end{align}
is guaranteed to be 1-Lipschitz, 
where $R(x)$ is a channel-wise rescaling
that multiplies each channel $c\in\{1,\dots,c_I\}$ of the input by
\begin{align}
    d_c &= \Big( \sum_{i, j}\sum_{a=1}^{c_\text{I}} \Big| 
        \sum_{b = 1}^{c_\text{O}}P^{(a, b)} * P^{(c, b)} \Big|_{i, j} \Big)^{-1/2}.
\end{align}
We can equivalently write $f$ as $f(x) = W * x + b$,
where $W=P*D$ with $D \in \mathbb{R}^{1 \times 1 \times c_\text{I} \times c_\text{I}}$ 
given by $D_{1, 1}^{(c, c)} = d_c$,
and $ D_{1, 1}^{(c_1, c_2)} = 0$ for $c_1 \ne c_2$.
\end{restatable}


The proof consists of an explicit derivation of the Jacobian of the convolution operation as a linear map, followed by an application of Theorem 1. The main step is the explicit demonstration that the diagonal rescaling matrix can in fact be 
bounded by 
a per-channel multiplication with the result of a self-convolution of the convolution kernel. The details can be found in the \supp.

\paragraph{Discussion.}
We now discuss some favorable properties of  Lemma~\ref{lem:convolutional}.
First, as in the fully-connected case, the rescaling operation again can be viewed either as acting on the inputs, or as acting on the parameter matrix. Therefore, the convolutional layer also combines the properties of high capacity and no overhead at prediction time. In fact, for $1\times 1$ convolutions, the construction of Lemma~\ref{lem:convolutional} reduces to the fully-connected situation of Lemma~\ref{lem:fullyconnected}.

Second, computing the scaling factors is efficient, because the necessary operations scale only with the size of the convolution kernel regardless of the image size. 
The rescaling preserves the structure of the convolution kernel, \eg sparsity patterns. In particular, this means that constructs such as dilated 
convolutions are automatically covered by Lemma~\ref{lem:convolutional} as well, as these can be expressed as ordinary convolutions with specific zero entries. 

Furthermore, Lemma~\ref{lem:convolutional} requires no strong assumption on the padding type, and works as long as the padding itself is 1-Lipschitz.
Also, the computation of the scale factors is easy to implement in all common deep learning frameworks using batch-convolution operations with the input channel dimension taking the role of the batch dimension.

\section{Experiments}
We compare our method to related work
in the context of
\emph{\cra}, 
where the goal is to solve an
image classification task in a way that provably 
prevents \emph{adversarial examples}~\cite{Szegedy_2014_ICLR}.
Specifically, we consider an input $x$ 
as \emph{certifiably robustly classified} by a model
under input perturbations up to size $\epsilon$,
if $x + \delta$ is correctly classified 
for all $\delta$ with $\| \delta \| \le \epsilon$.
Then the \emph{\cra} of a classifier is the
proportion of the test set that is certifiably robustly classified.

Consider a function $f$ that generates a score for each class.
Define the \emph{margin} of $f$ at input $x$ with correct label $y$
as 
\begin{align}
    M_f(x) &= \big[ f(x)_y - \max_{i \ne y} f(x)_i\big]_+ \qquad\text{with}\qquad 
[\cdot]_+ = \max\{\cdot,0\}.
\end{align}
Then the induced classifier, $C_f(x)=\operatorname{argmax}_{i} f(x)_i$, certifiably robustly classifies an input $x$
if $M_f(x) > \sqrt{2}L\epsilon$,
where $L$ is the Lipschitz constant $L$ of $f$.
This relation can be used to efficiently determine (a lower bound to) the \cra 
of Lipschitz networks~\cite{Tsuzuku_2018_NIPS}.

Following prior work in the field, we conduct experiments 
that evaluate the \cra for different thresholds, $\epsilon$, on the CIFAR-10 as well as the CIFAR-100 dataset.
We also provide ablation studies that illustrate that
\method can be used in a variety of network architectures,
and that it indeed learns matrices that are approximately 
orthogonal.
In the following we describe our experimental setup. Further details can be found in the \supp. 
All hyperparameters were determined on validation sets.

\myparagraph{Architecture:}
Our main model architecture is loosely inspired by the 
\textit{Conv\-Mixer} architecture~\cite{Trockman_2022_Arxiv}:
we first subdivide the input image into $4\times 4$ patches, 
which are processed by 14 convolutional layers, most
of kernel size $3\times 3$.
This is followed by 14 fully connected layers.
%
We report results for three different model sizes, we will refer to the models as 
\method-Small, \method-Medium and \method-Large.
We use the MaxMin activation 
function~\cite{Chernodub_2016_Arxiv,Anil_2019_ICML}.
%
The full architectural details can be found in Table  \ref{table:architecture_patchwise}.

Other network architectures are discussed in an ablation study in Section~\ref{sec:ablation}.

\newcommand{\pwconv}{\aolconv}
\newcommand{\pwfc}{\aolfc}

\setlength{\tabcolsep}{4pt}
\begin{table}[t]
\begin{center}
\caption{Patchwise architecture. 
For all layers we use zero padding to keep 
the size the same. 
For \method-Small we set $w$ to $16$,
and we choose $w=32$ and $w=48$ for \method-Medium and \method-Large.
Furthermore, $l$ is the number of classes, 
and $l=10$ for CIFAR-10 and $l=100$ for CIFAR-100.
\emph{Concatenation Pooling} stacks all the inputs into
a single vector, and
\emph{First channels} just selects the first channels and ignores the rest.
}
\label{table:architecture_patchwise}
\begin{tabular}{l|l|l|l|l|l}
\hline\noalign{\smallskip}
Layer name  & Kernel size & Stride & Activation & Output size & Amount \\
\noalign{\smallskip}
\hline
\noalign{\smallskip}

Concatenation Pooling & $4 \times 4$ & $4 \times 4$ & - & $8 \times 8 \times 48$ & $1$ \\
\pwconv & $1 \times 1$ & $1 \times 1$ & \mm & $8 \times 8 \times 192$ & $1 \atimes$ \\
\pwconv & $3 \times 3$ & $1 \times 1$ & \mm & $8 \times 8 \times 192$ & $12 \atimes$ \\
\pwconv & $1 \times 1$ & $1 \times 1$ & None & $8 \times 8 \times 192$ & $1 \atimes$ \\
First Channels & - & - & - & $8 \times 8 \times w$ & $1 \atimes$ \\
Flatten & - & - & - & $64w$ & $1 \atimes$ \\
\pwfc & - & - & \mm & $64w$ & $13 \atimes$ \\
\pwfc & - & - & None & $64w$ & $1 \atimes$ \\
First Channels & - & - & - & $ l $ & $1 \atimes$ \\
%
%
\noalign{\smallskip}
\hline
\end{tabular}
\end{center}
\end{table}
\setlength{\tabcolsep}{1.4pt}

\myparagraph{Initialization:}
In order to ensure stable training
we initialize the parameter matrices so that our bound is tight. 
In particular, for layers preserving the size 
between input and output
(e.g. the $3 \times 3$ convolutions)
we initialize the parameter matrix so that the Jacobian is the identity matrix.
For any other layers we initialize the parameter matrix 
so that it has random orthogonal columns. 

\myparagraph{Loss function:}
In order to train the network to achieve good \cra{} we want the score
of the correct class to be bigger than any other score by a margin.
We use a loss function similar to the one proposed for \textit{Lipschitz-margin training} \cite{Tsuzuku_2018_NIPS}
with a temperature parameter that helps encouraging a margin during training.
Our loss function takes as input the model's logit vector, $\vec{s}$, as well as a one-hot encoding $\vec{y}$ of the true label as input, and is given by
\begin{align}
    \mathcal{L}(\vec{s}, \vec{y}) = \operatorname{crossentropy}\Big( \vec{y}, \ 
            \operatorname{softmax}\big(\frac{\vec{s} - u\vec{y}}{t}\big) \Big) t,
            \label{eq:loss}
\end{align}
for some offset $u$ and some temperature $t$.
For our experiments, we use $u=\sqrt{2}$, 
which encourages the model to learn to classify the training data 
certifiably robustly to perturbations of norm $1$.
Furthermore we use temperature $t=1/4$, which 
causes the gradient magnitude to stay close to $1$
as long as a training example is classified 
with margin less than $1/2$.

\myparagraph{Optimization:}
We minimize the loss function~\eqref{eq:loss} using SGD with Nesterov momentum of 0.9 for 1000 epochs. The batch size is 250.
The learning rate starts at $10^{-3}$ and is reduced by a factor of $10$ at epochs $900, 990$ and $999$.
As data augmentation we use spatial transformations (rotations and flipping) as well as some color 
transformation. The details are provided in the \supp. 
For all \method layers we also use weight decay with coefficient $5 \times 10^{-4}$.


\section{Results}
The main results can be found in Table \ref{table:cifar10_acc} and Table \ref{table:cifar100_acc},
where we compare the \cra of our method to those reported in previous works on orthogonal networks 
and other networks with bounded Lipschitz constant.
For methods that are presented in multiple variants, such as different networks depths, we include the variant for which the authors list results for large values of $\epsilon$.

\newcommand{\unpublished}[1]{\textit{#1}}

\begin{table}[t]
\setlength{\tabcolsep}{4pt}
\begin{center}
\caption{Experimental results: 
robust image classification on CIFAR-10 for \method and methods from the literature.
\tabledesc
Results for concurrent unpublished works (ECO and SOC with Householder activations) 
are printed in italics. \emph{Standard CNN} refers to our implementation of a simple 
convolutional network trained without enforcing any robustness, for details see the \supp.
}
\label{table:cifar10_acc}
\begin{tabular}{l|l|llll}
\hline\noalign{\smallskip}
\fulltableheader
\noalign{\smallskip}
\hline
\noalign{\smallskip}
Standard CNN        & 83.4\%  & 0\%      & 0\%       & 0\%       & 0\%  \\
\noalign{\smallskip}
\hline
\noalign{\smallskip}
BCOP Large \cite{Li_2019_NIPS}          & 72.2\%  & 58.3\%  & -       & -       & -  \\
GloRo 6C2F \cite{Leino_2021_ICML}       & 77.0\%  & 58.4\%  & -       & -       & -  \\
Cayley Large \cite{Trockman_2021_ICLR}  & 75.3\%  & 59.2\%  & -       & -       & -  \\
SOC-20 \cite{Singla_2021_ICML}          & 76.4\%  & 61.9\%  & -       & -       & -  \\
%
\it{ECO-30 \cite{Yu_2022_ICLR}} & \it{72.5\%} & \it{55.5\%} & -       & -       & -  \\
\it{SOC-15 + CR \cite{Singla_2022_ICLR}}  
                            & \it{76.4\%}  & \it{63.0\%}  & \it{48.5\%}  & \it{35.5\%}  & -  \\
\noalign{\smallskip}
\hline
\noalign{\smallskip}
AOL-Small   & 69.8\% & 62.0\% & 54.4\% & 47.1\% & 21.8\% \\
AOL-Medium  & 71.1\% & 63.8\% & 56.1\% & 48.6\% & 23.2\% \\
AOL-Large   & 71.6\% & 64.0\% & 56.4\% & 49.0\% & 23.7\% \\
\noalign{\smallskip}
\hline
\noalign{\smallskip}
\end{tabular}
\end{center}
\setlength{\tabcolsep}{1.4pt}
\end{table}

\begin{table}[t]
\setlength{\tabcolsep}{4pt}
\begin{center}
\caption{Experimental results: 
robust image classification on CIFAR-100 for \method and methods from the literature.
\tabledesc
Results for concurrent unpublished works (ECO and SOC with Householder activations) 
are printed in italics.
}
\label{table:cifar100_acc}
\begin{tabular}{l|c|cccc}
\hline\noalign{\smallskip}
\fulltableheader
\noalign{\smallskip}
\hline
\noalign{\smallskip}
SOC-30 \cite{Singla_2021_ICML}  & 43.1\%  & 29.2\%  & -       & -      & - \\
%
\it{ECO-30 \cite{Yu_2022_ICLR}} & \it{40.0\%} & \it{25.4\%} & -       & -       & -  \\
\it{SOC+CR \cite{Singla_2022_ICLR}}   
                                & \it{47.8\%}  & \it{34.8\%}  & \it{23.7\%}  & \it{15.8\%}  & - \\
\noalign{\smallskip}
\hline
\noalign{\smallskip}
AOL-Small   & 42.4\% & 32.5\% & 24.8\% & 19.2\% & 6.7\% \\
AOL-Medium  & 43.2\% & 33.7\% & 26.0\% & 20.2\% & 7.2\% \\
AOL-Large   & 43.7\% & 33.7\% & 26.3\% & 20.7\% & 7.8\% \\
\noalign{\smallskip}
\hline
\noalign{\smallskip}
\end{tabular}
\end{center}
\setlength{\tabcolsep}{1.4pt}
\end{table}

The table shows that our proposed method achieves 
results comparable with the current state-of-the-art. 
%
For small robustness thresholds, it achieves 
\cra slightly higher than published earlier methods, 
and compareable to the one reported in 
a concurrent preprint~\cite{Singla_2022_ICLR}.
Focusing on (more realistic) medium or higher 
robustness thresholds, \method achieves \cra 
comparable to or even higher than all other 
methods.
As a reference for future work, we also report 
values for an even higher robustness threshold 
than what appeared in the literature so far, $\epsilon=1$. 

Another observation is that on the CIFAR-10 dataset
the clean accuracy of \method is somewhat below 
other methods. We attribute this to the fact that 
we mainly focused our training towards high robustness. %
The accuracy-robustness trade-off can in fact 
be influenced by the choice of margin at 
training time, see our ablation study in Section~\ref{sec:ablation}. 

\subsection{Ablation Studies}\label{sec:ablation}
In this section we report on a number of 
ablation studies that shed light on specific
aspect of \method.

\myparagraph{Generality:}
One of the main advantages of \method is that it
is not restricted to a specific architecture 
or a specific layer type.
To demonstrate this, we present additional 
experiments for a broad range of other architectures.
\method-FC consists simply of $9$ fully connected layers.
\method-STD resembles a standard convolutional 
architecture, where the number of channels doubles whenever the resolution is reduced.
\method-ALT is another convolutional architecture that keeps the number of activations
constant wherever possible in the network.
\method-DIL resembles the architectures used in~\cite{Yu_2022_ICLR} in that it uses large dilated convolutions instead of small local ones. It also uses circular padding.
The details of the architectures are provided in the \supp. 

The results (shown in the \supp)
confirm that for any of these architectures, we can
train \method-based Lipschitz networks and achieve 
\cra comparable to the results of earlier specialized methods.

\myparagraph{Approximate Orthogonality:}
As a second ablation study, we demonstrate that 
\method indeed learns almost-orthogonal weight 
matrices, thereby justifying its name.
In order to do that, we evaluate $J^\top\!J$ for $J$ the Jacobian of an \method convolution,
and visualize it in Figure \ref{fig:jacobian}. 
More detailed results including a comparison 
to standard training are provided in the \supp.

\begin{figure}[t]
\centering
\includegraphics[width=\textwidth]{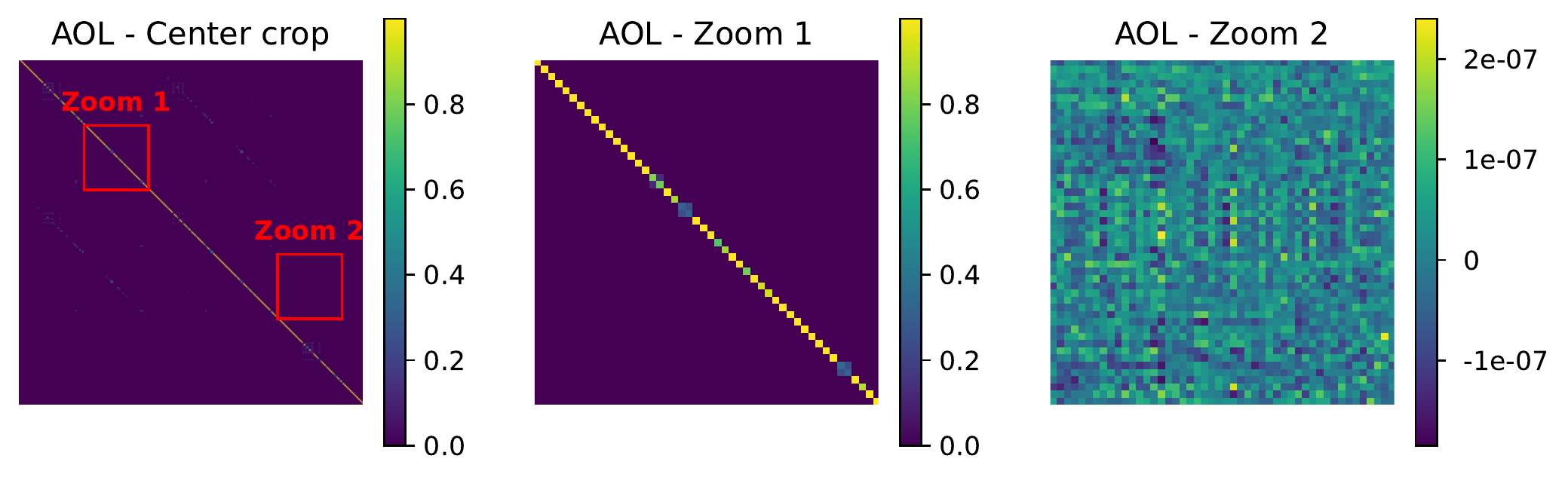}
\caption{Evaluation of the orthogonality of the trained model.
We consider the third layer of the \method-Small model.
It is a $3 \times 3$ convolutions with input size $8 \times 8 \times 192$.
We show a center crop of $J^\top J$, for $J$ the Jacobian, as well as two further crops.
Note that most diagonal elements are close to $1$,
and most off-diagonal elements are very close to $0$
(note the different color scale in the third subplot).
This confirms that \method did indeed learn an almost-orthogonal weight matrix.
Best viewed in color and zoomed in.
}
\label{fig:jacobian}
\end{figure}

\setlength{\tabcolsep}{4pt}
\begin{table}
\begin{center}
\caption{Experimental results for \method-Small 
for different value of
$u$ and $t$ in the loss function
in Equation~\eqref{eq:loss}.
\tabledesc
}
\label{table:acc_rob_tradeoff}
\begin{tabular}{c|c|c|cccc}
\noalign{\smallskip}
\hline
\noalign{\smallskip}
$u$  & $t$  & Standard  & \multicolumn{4}{c}{\Cra} \\
     &      & Accuracy  & $\epsilon=\frac{36}{255}$ & $\epsilon=\frac{72}{255}$ & $\epsilon=\frac{108}{255}$ & $\epsilon=1$ \\
\noalign{\smallskip}
\hline
\noalign{\smallskip}
%
$\sqrt{2}/16$ & $ 1/64 $ & 79.8\% & 45.3\% & 16.7\% &  3.3\% &  0.0\% \\
$\sqrt{2}/4$  & $ 1/16 $ & 77.4\% & 63.0\% & 47.6\% & 33.0\% &  2.5\% \\
$\sqrt{2}$    & $ 1/4  $ & 70.4\% & 62.6\% & 55.0\% & 47.9\% & 22.2\% \\
$4 \sqrt{2}$  & $ 1    $ & 59.8\% & 55.5\% & 50.9\% & 46.5\% & 30.8\% \\
$16 \sqrt{2}$ & $ 4    $ & 48.2\% & 45.2\% & 42.2\% & 39.4\% & 28.6\% \\
\noalign{\smallskip}
\hline
\noalign{\smallskip}
\end{tabular}
\end{center}
\end{table}
\setlength{\tabcolsep}{1.4pt}

\myparagraph{Accuracy-Robustness Tradeoff:}
The loss function in Equation~\eqref{eq:loss} allows 
trading off between clean accuracy and \cra by changing 
the size of the enforced margin. 
We demonstrate this by an ablation study that 
varies the offset parameter $u$ in the loss function,
and also scales $t$ proportional to $u$.

The results can be found in Table~\ref{table:acc_rob_tradeoff}.
One can see that using a small margin allows us to train
an \method Network with high clean accuracy, but decreases the \cra for larger input perturbations,
whereas choosing a higher offset allows us to reach state-of-the-art accuracy
for larger input perturbations.
Therefore, varying this offset gives us an easy way to
prioritize the measure that is important for a specific problem.

\section{Conclusion}
In this work, we proposed \method, a method for 
constructing deep networks that have Lipschitz 
constant of at most $1$ and therefore 
are robust against small changes in the input data.
Our main contribution is a rescaling technique
for network layers that ensures them to be  
1-Lipschitz.
It can be computed and trained efficiently,
and is applicable to fully-connected and various types of convolutional layers.
Training with the rescaled layers leads to 
weight matrices that are almost orthogonal 
without the need for a special parametrization 
and computationally costly orthogonalization 
schemes. 
We present experiments and ablation studies in the context
of image classification with certified robustness.
They show that \method-networks achieve results 
comparable with methods that explicitly enforce 
orthogonalization, while offering the simplicity
and flexibility of earlier bound-based approaches.


\bibliographystyle{unsrt}
\bibliography{references}  

\begin{thebibliography}{10}

\bibitem{Szegedy_2014_ICLR}
Christian Szegedy, Wojciech Zaremba, Ilya Sutskever, Joan Bruna, Dumitru Erhan,
  Ian~J. Goodfellow, and Rob Fergus.
\newblock Intriguing properties of neural networks.
\newblock In {\em International Conference on Learning Representations (ICLR)},
  2014.

\bibitem{chakraborty2018adversarial}
Anirban Chakraborty, Manaar Alam, Vishal Dey, Anupam Chattopadhyay, and Debdeep
  Mukhopadhyay.
\newblock Adversarial attacks and defences: A survey.
\newblock {\em arXiv preprint arXiv:1810.00069}, 2018.

\bibitem{serban2020adversarial}
Alex Serban, Erik Poll, and Joost Visser.
\newblock Adversarial examples on object recognition: A comprehensive survey.
\newblock {\em ACM Computing Surveys (CSUR)}, 53(3):1--38, 2020.

\bibitem{xu2020adversarial}
Han Xu, Yao Ma, Hao-Chen Liu, Debayan Deb, Hui Liu, Ji-Liang Tang, and Anil~K
  Jain.
\newblock Adversarial attacks and defenses in images, graphs and text: A
  review.
\newblock {\em International Journal of Automation and Computing},
  17(2):151--178, 2020.

\bibitem{virmaux2018lipschitz}
Aladin Virmaux and Kevin Scaman.
\newblock Lipschitz regularity of deep neural networks: analysis and efficient
  estimation.
\newblock In {\em Conference on Neural Information Processing Systems
  (NeurIPS)}, 2018.

\bibitem{arjovsky2017wasserstein}
Martin Arjovsky, Soumith Chintala, and L{\'e}on Bottou.
\newblock Wasserstein generative adversarial networks.
\newblock In {\em International Conference on Machine Learing (ICML)}, 2017.

\bibitem{gulrajani2017improved}
Ishaan Gulrajani, Faruk Ahmed, Martin Arjovsky, Vincent Dumoulin, and Aaron~C
  Courville.
\newblock Improved training of {W}asserstein {GANs}.
\newblock In {\em Conference on Neural Information Processing Systems
  (NeurIPS)}, 2017.

\bibitem{miyato2018spectral}
Takeru Miyato, Toshiki Kataoka, Masanori Koyama, and Yuichi Yoshida.
\newblock Spectral normalization for generative adversarial networks.
\newblock In {\em International Conference on Learning Representations (ICLR)},
  2018.

\bibitem{Cisse_2017_ICML}
Moustapha Ciss{\'e}, Piotr Bojanowski, Edouard Grave, Yann~N. Dauphin, and
  Nicolas Usunier.
\newblock Parseval networks: Improving robustness to adversarial examples.
\newblock In {\em International Conference on Machine Learing (ICML)}, 2017.

\bibitem{Wang_2020_CVPR}
Jiayun Wang, Yubei Chen, Rudrasis Chakraborty, and Stella~X. Yu.
\newblock Orthogonal convolutional neural networks.
\newblock In {\em Conference on Computer Vision and Pattern Recognition
  (CVPR)}, 2020.

\bibitem{Gouk_2021_ML}
Henry Gouk, Eibe Frank, Bernhard Pfahringer, and Michael~J Cree.
\newblock Regularisation of neural networks by enforcing {Lipschitz}
  continuity.
\newblock {\em Machine Learning}, 110(2):393--416, 2021.

\bibitem{Tsuzuku_2018_NIPS}
Yusuke Tsuzuku, Issei Sato, and Masashi Sugiyama.
\newblock Lipschitz-margin training: Scalable certification of perturbation
  invariance for deep neural networks.
\newblock In {\em Conference on Neural Information Processing Systems
  (NeurIPS)}, 2018.

\bibitem{Leino_2021_ICML}
Klas Leino, Zifan Wang, and Matt Fredrikson.
\newblock Globally-robust neural networks.
\newblock In {\em International Conference on Machine Learing (ICML)}, 2021.

\bibitem{Li_2019_NIPS}
Bai Li, Changyou Chen, Wenlin Wang, and Lawrence Carin.
\newblock Certified adversarial robustness with additive noise.
\newblock In {\em Conference on Neural Information Processing Systems
  (NeurIPS)}, 2019.

\bibitem{Anil_2019_ICML}
Cem Anil, James Lucas, and Roger~B. Grosse.
\newblock Sorting out {L}ipschitz function approximation.
\newblock In {\em International Conference on Machine Learing (ICML)}, 2019.

\bibitem{Huang_2020_CVPR}
Lei Huang, Li~Liu, Fan Zhu, Diwen Wan, Zehuan Yuan, Bo~Li, and Ling Shao.
\newblock Controllable orthogonalization in training {DNNs}.
\newblock In {\em Conference on Computer Vision and Pattern Recognition
  (CVPR)}, 2020.

\bibitem{Trockman_2021_ICLR}
Asher Trockman and J.~Zico Kolter.
\newblock Orthogonalizing convolutional layers with the {Cayley} transform.
\newblock In {\em International Conference on Learning Representations (ICLR)},
  2021.

\bibitem{Singla_2021_ICML}
Sahil Singla and Soheil Feizi.
\newblock Skew orthogonal convolutions.
\newblock In {\em International Conference on Machine Learing (ICML)}, 2021.

\bibitem{Yu_2022_ICLR}
Tan Yu, Jun Li, YUNFENG CAI, and Ping Li.
\newblock Constructing orthogonal convolutions in an explicit manner.
\newblock In {\em International Conference on Learning Representations (ICLR)},
  2022.
\newblock (to appear).

\bibitem{Bjorck_1971_SIAM}
{\AA}.~Bj{\"o}rck and C.~Bowie.
\newblock An iterative algorithm for computing the best estimate of an
  orthogonal matrix.
\newblock {\em SIAM Journal on Numerical Analysis}, 1971.

\bibitem{Cayley_1846_JRAM}
Arthur Cayley.
\newblock About the algebraic structure of the orthogonal group and the other
  classical groups in a field of characteristic zero or a prime characteristic.
\newblock {\em Journal für die reine und angewandte Mathematik}, 1846.

\bibitem{Sedghi_2019_ICLR}
Hanie Sedghi, Vineet Gupta, and Philip~M. Long.
\newblock The singular values of convolutional layers.
\newblock In {\em International Conference on Learning Representations (ICLR)},
  2019.

\bibitem{Trockman_2022_Arxiv}
Asher Trockman and J~Zico Kolter.
\newblock Patches are all you need?
\newblock {\em arXiv preprint arXiv:2201.09792}, 2022.

\bibitem{Chernodub_2016_Arxiv}
Artem~N. Chernodub and Dimitri Nowicki.
\newblock Norm-preserving orthogonal permutation linear unit activation
  functions (oplu).
\newblock {\em CoRR}, 2016.

\bibitem{Singla_2022_ICLR}
Sahil Singla, Surbhi Singla, and Soheil Feizi.
\newblock Improved deterministic $l_2$ robustness on {CIFAR}-10 and
  {CIFAR}-100.
\newblock In {\em International Conference on Learning Representations (ICLR)},
  2022.
\newblock (to appear).

\end{thebibliography}

\clearpage
\appendix

\newcommand{\cI}{c_\text{I}}
\newcommand{\cO}{c_\text{O}}

\section{Appendix}

\subsection{Proof of Lemma \ref{lem:convolutional}}
We will proof the Lemma \ref{lem:convolutional} here. Recall

\convlemma*

\begin{proof}

For the proof we will assume \emph{maximal} padding of the input: 
We pad an input $x \in \mathbb{R}^{n \times n \times c_\text{I}}$ (with values independent of $x$)
to a size of $(n + 2k-2) \times (n + 2k-2) \times \cI$,
and then apply the convolution to the padded input.
Then we obtain an output of size $(n+k-1) \times (n+k-1) \times \cO$.
We will derive a rescaling of the input that ensures 
this convolution has a Lipschitz constant of $1$.
Then, any convolution with a different kind of padding 
(such as \emph{same} size or \emph{valid})
can be considered as first doing a maximally padded convolution,
followed by a center cropping operation. 
Since a cropping operation also has a Lipschitz constant of $1$, 
this also shows that convolutional layers 
with a different kind of padding
have a Lipschitz constant of $1$.

Denote by $\tilde{x}$ the padded version of input $x$,
with $\tilde{x}_{i+k-1, j+k-1} = x_{i, j}$.
Then, the multi-channel, maximally padded convolution 
with a convolutional kernel
$P \in\mathbb{R}^{k \times k \times c_I \times c_O}$ is given by

\begin{align}
    [P * x]^b_{i,j} = \sum_{p=0}^{k-1} \sum_{q=0}^{k-1} \sum_{a=1}^{c_I} 
        P^{(a,b)}_{p, q} \tilde{x}^a_{i+p,j+q}\label{eq:convolutiondef2},
\end{align}
for $1 \le i, j \le n+k-1$ and $1 \le b \le \cO$.

We now consider the Jacobian $J$ of the linear map 
(from unpadded input to output)
defined by Equation \ref{eq:convolutiondef2}.
It is a matrix of size 
$(n+k-1)^2\cO \times n^2\cI$, with entries given by
\begin{align}
    J^{(b,a)}_{(i_2,j_2),(i_1,j_1)} 
    = P^{(a,b)}_{(i_1-i_2+k-1),(j_1-j_2+k-1)},
\end{align}
for $1 \le i_1, i_2 \le n$, 
$1 \le i_2, j_2 \le n+k-1$, 
$1 \le a \le \cI$ and 
$1 \le b \le \cO$.
Here, we define $P^{(a,b)}_{p, q} = 0$ 
unless $0 \le p < n$ and $0 \le q < n$.

We can use that to obtain an expression for $J^\top J$ (writing $m=n+k-1$):
\begin{align}
&[J^\top J]^{(a_1, a_2)}_{(i_1,j_1),(i_2,j_2)} \\
    &\qquad= \sum_{i=1}^{m} \sum_{j=1}^{m} \sum_{b=1}^{c_\text{O}}
        J^{(b, a_1)}_{(i,j), (i_1,j_1)}J^{(b, a_2)}_{(i,j), (i_2,j_2)}
    \\
    &\qquad= \sum_{i=1}^{m} \sum_{j=1}^{m} \sum_{b=1}^{c_\text{O}}
        P^{(a_1, b)}_{(i_1-i+k-1),(j_1-j+k-1)} P^{(a_2, b)}_{(i_2-i+k-1),(j_2-j+k-1)}
    \\
    &\qquad= \sum_{i=0}^{k-1} \sum_{j=0}^{k-1} \sum_{b=1}^{c_\text{O}}
        P^{(a_1, b)}_{i, j} P^{(a_2, b)}_{(i+i_2-i_1),(j+j_2-j_1)}
    \\
    &\qquad= \left( \sum_{b=1}^{c_\text{O}} P^{(a_1, b)} * P^{(a_2, b)} \right)_{i_2 - i_1, j_2 - j_1}.
    \label{eq:conv}
\end{align}

We can now apply Theorem~\ref{thm} (from the main paper) with $P=J$ together with 
Equation \eqref{eq:conv} in order to obtain the necessary rescaling: 
In order to guarantee the convolution 
to have Lipschitz constant $1$,
we need to multiply input $x^{(c)}_{i_2, j_2}$ by
\begin{align}
    &\left( \sum_{i_1=1}^n \sum_{j_1=1}^n \sum_{a_1=1}^{\cI} 
        \left| \sum_{b=1}^{c_\text{O}} P^{(a_1, b)} * P^{(c, b)} \right|_{i_2 - i_1, j_2 - j_1}
        \right)^{-1/2}.
\end{align}
A lower bound of this expression 
(that is tight for most values of $i_2$ and $j_2$)
is given by
\begin{align} \label{eq:final}
    &\left( \sum_{i=-k+1}^{k} \sum_{j=-k+1}^{k} \sum_{a=1}^{\cI} 
        \left| \sum_{b=1}^{c_\text{O}} P^{(a, b)} * P^{(c, b)} \right|_{i, j}
        \right)^{-1/2}.
\end{align}
This value is independent of both $i_2$ and $j_2$, 
so this completes our proof for maximally padded convolutions. 
This also implies that convolutions
with less padding have a Lipschitz constant of $1$ 
when the input is rescaled as described.
\qed
\end{proof}

Note that this proof requires padding independent of the input.
However, for example for cyclic padding, 
the Jacobian is a doubly circuland matrix,
and a very similar proof shows that our rescaling still works.

Also note that a result similar to equation \eqref{eq:conv},
relating $J^\top J$ to a self-convolution,
has been observed before by \cite{Wang_2020_CVPR}.

\subsection{Comparison of the orthogonality}
We will present a comparison of $J^\top J$ for $J$ the Jacobian of different layers.
We consider three architectures: 
A standard convolutional architecture with standard convolutions (see Section \ref{std}),
the same architectures with \method convolutions 
and (theoretically) an architecture with perfectly orthogonal Jacobian.
We pick the third layer of this architecture.
It is a convolution with kernel of size $3 \times 3 \times 32 \times 32$,
and input as well as output of size $32 \times 32 \times 32$.
This results in a Jacobian $J$ of size $32\,768 \times 32\,768$,
and we calculate and visualize 
the values of a center crop of size $288 \times 288$ of the matrix $J^\top J$.
(See Figure \ref{fig:jacobian_compared}.)

\begin{figure}[t]
\centering
\includegraphics[width=.8\textwidth]{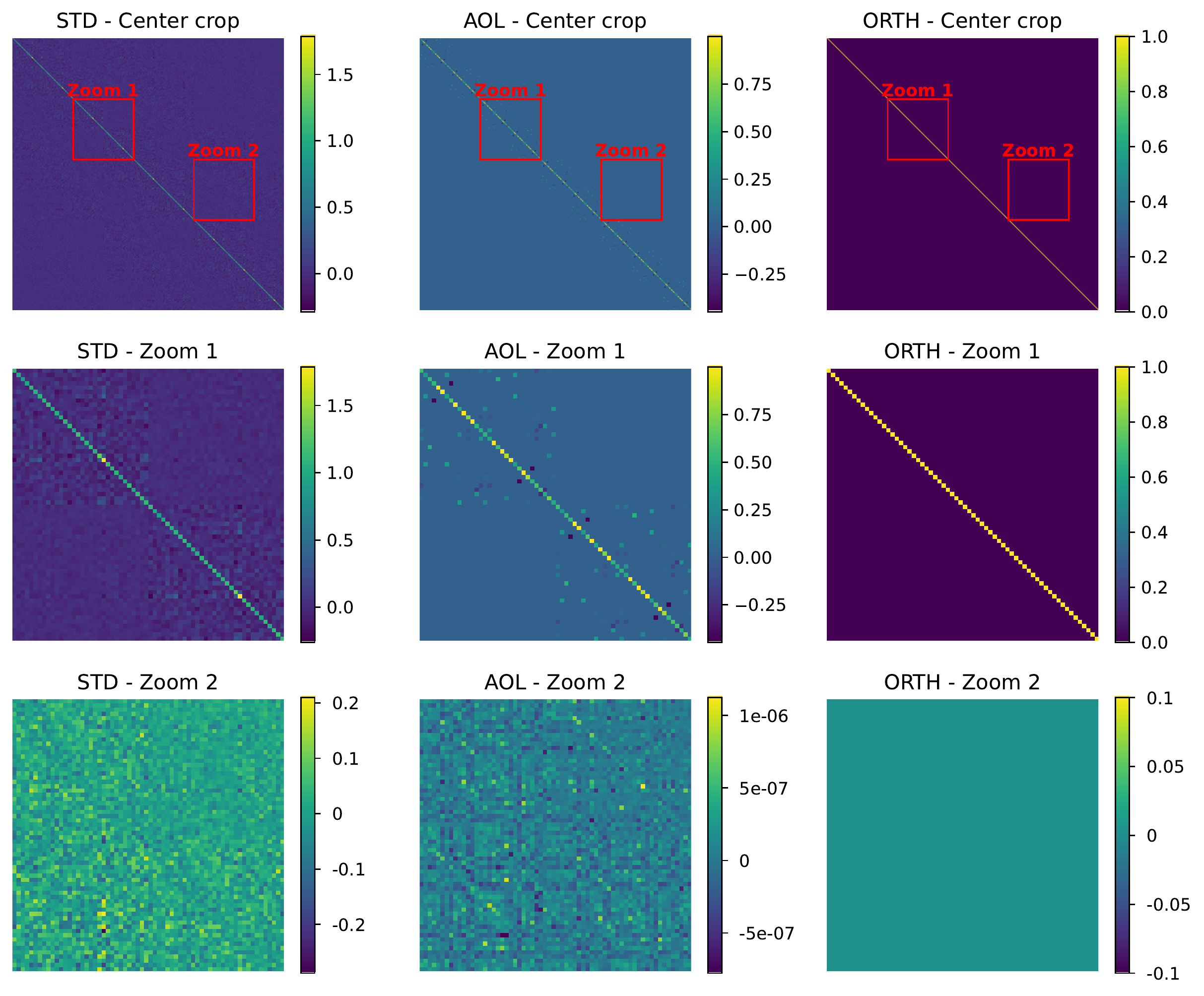}
\caption{Evaluation of orthogonality of trained models.
The first row shows a center crop of $J^\top J$, where $J$ is the Jacobian of the layer,
as well as the location of the other two crops.
Those are shown in the second and third row.
\textbf{Left column:} Standard convolutional layer (STD).
\textbf{Center column:} \method-STD (AOL),
\textbf{Right column:} (Perfectly) orthogonal layer (ORTH).
Note the different color scales of different subplots.
Best viewed in color and zoomed in.
}
\label{fig:jacobian_compared}
\end{figure}

One can see that for our \method-STD architecture most off-diagonal elements
are very close to $0$, whereas for the standard architecture they are not.
Interestingly, for our architecture there are some off-diagonal elements
that are clearly non-zero. 
This shows that the learning 
resulted in a few column pairs
not being orthogonal, in line with our claim of almost-orthogonality.

\subsection{Accuracies of different architectures}
The results for different architectures using our
proposed AOL layers are in Table~\ref{table:ablation_accuracies}.

\setlength{\tabcolsep}{4pt}
\begin{table}[h]
\begin{center}
\caption{Experimental results for different architectures 
using \method on CIFAR-10.
\tabledesc
}
\label{table:ablation_accuracies}
\begin{tabular}{l|l|llll}
\hline\noalign{\smallskip}
\fulltableheader
\noalign{\smallskip}
\hline
\noalign{\smallskip}
AOL-FC      & 67.9\%  & 59.1\%  & 51.1\%  & 43.1\%  & 17.6\%  \\
AOL-STD     & 65.0\%  & 56.4\%  & 48.2\%  & 39.9\%  & 15.8\%  \\
AOL-ALT     & 68.4\%  & 60.3\%  & 52.4\%  & 44.8\%  & 19.9\%  \\
AOL-DIL     & 62.7\%  & 54.2\%  & 46.0\%  & 38.3\%  & 14.9\%  \\
\noalign{\smallskip}
\hline
\noalign{\smallskip}
\end{tabular}
\end{center}
\end{table}
\setlength{\tabcolsep}{1.4pt}

\subsection{Further Architectures} \label{Architectures}
We present the other architectures here that were used in the ablation studies.

\subsubsection{Standard CNN:} \label{std}
In this Architecture the number of channels doubles whenever the spatial size is decreased.
For details see Table \ref{table:architecture_std_conv}.

\newcommand{\stdconv}{Conv}

\setlength{\tabcolsep}{4pt}
\begin{table}
\begin{center}
\caption{A relatively standard convolutional architecture 
for CIFAR-10. 
For all layers we use zero padding to keep the size the same.
\emph{First channels} just selects the first channels and ignores the rest.
}
\label{table:architecture_std_conv}
\begin{tabular}{l|l|l|l|l|l|l}
\hline\noalign{\smallskip}
Layer name  & Filters & Kernel size & Stride & Activation & Output size & Amount \\
\noalign{\smallskip}
\hline
\noalign{\smallskip}
\stdconv & $32$ & $3 \times 3$ & $1 \times 1$ & \mm & $32 \times 32 \times 32$ & $4 \atimes$ \\
\stdconv & $64$ & $2 \times 2$ & $2 \times 2$ & None & $16 \times 16 \times 64$ & $1 \atimes$ \\
\stdconv & $64$ & $3 \times 3$ & $1 \times 1$ & \mm & $16 \times 16 \times 64$ & $4 \atimes$ \\
\stdconv & $128$ & $2 \times 2$ & $2 \times 2$ & None & $8 \times 8 \times 128$ & $1 \atimes$ \\
\stdconv & $128$ & $3 \times 3$ & $1 \times 1$ & \mm & $8 \times 8 \times 128$ & $4 \atimes$ \\
\stdconv & $256$ & $2 \times 2$ & $2 \times 2$ & None & $4 \times 4 \times 256$ & $1 \atimes$ \\
\stdconv & $256$ & $3 \times 3$ & $1 \times 1$ & \mm & $4 \times 4 \times 256$ & $4 \atimes$ \\
\stdconv & $512$ & $2 \times 2$ & $2 \times 2$ & None & $2 \times 2 \times 512$ & $1 \atimes$ \\
\stdconv & $512$ & $3 \times 3$ & $1 \times 1$ & \mm & $2 \times 2 \times 512$ & $4 \atimes$ \\
\stdconv & $1024$ & $2 \times 2$ & $2 \times 2$ & None & $1 \times 1 \times 1024$ & $1 \atimes$ \\
\stdconv & $1024$ & $1 \times 1$ & $1 \times 1$ & None & $1 \times 1 \times 1024$ & $1 \atimes$ \\
First Channels & - & - & - & - & $1 \times 1 \times 10$ & $1 \atimes$ \\
Flatten & - & - & - & - & $10$ & $1 \atimes$ \\
\hline
\end{tabular}
\end{center}
\end{table}
\setlength{\tabcolsep}{1.4pt}

\subsubsection{\method-FC:} \label{fcc}
Architecture consisting of $9$ fully connected layers.
For details see Table \ref{table:architecture_fully_fully_connected}.

\newcommand{\ffclnm}{\aolfc}

\setlength{\tabcolsep}{4pt}
\begin{table}
\begin{center}
\caption{Fully Connected Architecture. 
\emph{First channels} just selects the first channels 
and ignores the rest.
}
\label{table:architecture_fully_fully_connected}
\begin{tabular}{l|l|l|l}
\hline\noalign{\smallskip}
Layer name & Activation & Output size & Amount \\
\noalign{\smallskip}
\hline
\noalign{\smallskip}
Flatten & - & 3072 & $1 \atimes$ \\
\ffclnm & \mm & 4096 & $8 \atimes$ \\
\ffclnm & None & 4096 & $1 \atimes$ \\
First Channels & - & $10$ & $1 \atimes$ \\
\hline
\end{tabular}
\end{center}
\end{table}
\setlength{\tabcolsep}{1.4pt}

\subsubsection{\method-STD:} \label{std_conv}
This architecture is identical to the one in Table \ref{table:architecture_std_conv},
only with standard convolutions replaced by \method convolutions.

\subsubsection{\method-ALT:} \label{conv}
Architecture that quadruples the number of channels whenever the spatial size is decreased
in order to keep the number of activations constant for the first few layers.
This is done until there are $1024$ channels, then the number of channels is kept at this value.
For details see Table \ref{table:architecture_convolutional}.

\newcommand{\lnmalt}{AOL Conv}

\setlength{\tabcolsep}{4pt}
\begin{table}
\begin{center}
\caption{Convolutional architecture. For all layers we use zero padding to keep the size the same. 
\emph{First channels} just selects the first channels and ignores the rest.
}
\label{table:architecture_convolutional}
\begin{tabular}{l|l|l|l|l|l|l}
\hline\noalign{\smallskip}
Layer name  & Filters & Kernel size & Stride & Activation & Output size & Amount \\
\noalign{\smallskip}
\hline
\noalign{\smallskip}
\lnmalt & $16$   & $2 \times 2$ & $2 \times 2$ & \mm & $16 \times 16 \times 16$  & $1 \atimes$ \\
\lnmalt & $16$   & $3 \times 3$ & $1 \times 1$ & \mm & $16 \times 16 \times 16$  & $4 \atimes$ \\
\lnmalt & $64$   & $2 \times 2$ & $2 \times 2$ & \mm & $8 \times 8 \times 64$    & $1 \atimes$ \\
\lnmalt & $64$   & $3 \times 3$ & $1 \times 1$ & \mm & $8 \times 8 \times 64$    & $4 \atimes$ \\
\lnmalt & $256$  & $2 \times 2$ & $2 \times 2$ & \mm & $4 \times 4 \times 256$   & $1 \atimes$ \\
\lnmalt & $256$  & $3 \times 3$ & $1 \times 1$ & \mm & $4 \times 4 \times 256$   & $4 \atimes$ \\
\lnmalt & $1024$ & $2 \times 2$ & $2 \times 2$ & \mm & $2 \times 2 \times 1024$  & $1 \atimes$ \\
\lnmalt & $1024$ & $1 \times 1$ & $1 \times 1$ & \mm & $2 \times 2 \times 1024$  & $4 \atimes$ \\
\lnmalt & $1024$ & $1 \times 1$ & $1 \times 1$ & None & $2 \times 2 \times 1024$ & $1 \atimes$ \\
First Channels & $256$ & -     & -            & -   & $2 \times 2 \times 256$    & $1 \atimes$ \\
\lnmalt & $1024$ & $2 \times 2$ & $2 \times 2$ & \mm & $1 \times 1 \times 1024$  & $1 \atimes$ \\
\lnmalt & $1024$ & $1 \times 1$ & $1 \times 1$ & \mm & $1 \times 1 \times 1024$  & $4 \atimes$ \\
\lnmalt & $1024$ & $1 \times 1$ & $1 \times 1$ & None & $1 \times 1 \times 1024$ & $1 \atimes$ \\
First Channels & $10$ & - & - & - & $1 \times 1 \times 10$ & $1 \atimes$ \\
Flatten & - & - & - & - & $10$ & $1 \atimes$ \\
\hline
\end{tabular}
\end{center}
\end{table}
\setlength{\tabcolsep}{1.4pt}

\subsubsection{\method-DIL} \label{dilated}
We use an architecture similar to the one used by ECO.
For that, we just replace each $3 \times 3$ convolution in the architecture
in Table \ref{table:architecture_std_conv} with a strided \method convolution.

\subsection{Data augmentation}
We use data augmentation in all our experiments.
We first do some color augmentation of the images, followed by some spatial transformations.
For the color augmentation, 
we first adjust the hue of the image (by a random factor with delta in $[-0.02, 0.02]$),
then we adjust the saturation of the image (by a factor in $[.3, 2]$),
then we adjust the brightness of the image (by a random factor with delta in $[-0.1, 0.1]$),
and finally we adjust the contrast of the image (by a factor in $[.5, 2]$).
After this we clip the pixel values so they are in $[0., 1.]$.
For the spatial transformations, we first apply a random rotation, with a maximal rotation of $5$ degrees.
Then we apply a random shift of up to $10\%$ of the image size,
and finally we flip the image with a probability of $50\%$.
We rely on the tensorflow layers \emph{RandomRotation} and \emph{RandomTranslation} for the spatial transformations,
and leave all hyperparameters such as the fill mode as the default values.
All hyperparameters specifying the amount of augmentation were chosen based on visual inspection of the
augmented training images.

\end{document}